\documentclass[11pt]{article}
\usepackage{mathrsfs,amsmath,amsfonts,amssymb,bm,bbm,eufrak,dsfont,pifont,amscd,
stmaryrd,euscript,amsthm,appendix,color,epsfig,xr,yfonts,tikz,verbatim}
\usepackage[round]{natbib}
\usepackage{graphicx,color,authblk}
\usepackage{amsfonts}
\usepackage{bbm}
\usepackage{amsmath}
\usepackage{tikz}
\usepackage{hyperref}
\newcommand{\bb}{\mathbb}

\newcommand{\Cal}{\mathcal}

\usepackage{hyperref}       
\usepackage{url}            
\usepackage{booktabs}       
\usepackage{amsfonts}       
\usepackage{nicefrac}       
\usepackage{microtype}      %
\newtheorem{thm}{Theorem}

\newtheorem{pro}{Proposition}
\newtheorem{rem}{Remark}

\newtheorem{appxthm}{Theorem}[section]
\newtheorem{appxpro}[appxthm]{Proposition}
\newtheorem{appxlem}[appxthm]{Lemma}
\newtheorem{assumption}{Assumption}
\newtheorem{example}{Example}
 \newcommand{\iprod}[1]{\left \langle #1 \right \rangle}
\newcommand{\bX}{\mathbf{X}}
\newcommand{\tf}{f_{| \bX}}

\newcommand{\sx}{S_{\bm{X}}}
\newcommand{\sxs}{S^*_{\bm{X}}}
\newcommand{\hs}{\mathcal{L}_2(\mathcal{H})}

\title{Gaussian Sketching yields a J-L Lemma in RKHS}

\author[1]{Samory Kpotufe\thanks{skk2175@columbia.edu}}
\author[2]{Bharath K. Sriperumbudur\thanks{bks18@psu.edu}}
\affil[1]{Department of Statistics, Columbia University}
\affil[2]{Department of Statistics, Pennsylvania State University}

\date{}
\begin{document}

\maketitle

\begin{abstract}
The main contribution of the paper is to show that Gaussian sketching of a kernel-Gram matrix $\bm K$ yields an operator whose 
counterpart in an RKHS $\cal H$, is a \emph{random projection} 
operator---in the spirit of Johnson-Lindenstrauss (J-L) lemma. To be precise, given a random matrix $Z$ with i.i.d. Gaussian entries, we show that a sketch $Z\bm{K}$ corresponds to a particular random operator in (infinite-dimensional) Hilbert space $\cal H$ that maps functions $f \in \cal H$ to a low-dimensional space $\bb R^d$, while preserving a weighted RKHS inner-product of the form $\langle f, g \rangle_{\Sigma} \doteq \langle f, \Sigma^3 g \rangle_{\cal H}$, where $\Sigma$ is the \emph{covariance} operator induced by the data distribution. In particular, under similar assumptions as in kernel PCA (KPCA), or kernel $k$-means (K-$k$-means), well-separated subsets of feature-space $\{K(\cdot, x): x \in \cal X\}$ remain well-separated after such operation, which suggests similar benefits as in KPCA and/or K-$k$-means, albeit at the much cheaper cost of a random projection. In particular, our convergence rates suggest that, given a large dataset $\{X_i\}_{i=1}^N$ of size $N$, we can build the Gram matrix $\bm K$ on a much smaller subsample of size $n\ll N$, so that the sketch $Z\bm K$ is very cheap to obtain and subsequently apply as a projection operator on the original data $\{X_i\}_{i=1}^N$.  
We verify these insights empirically on synthetic data, and on real-world clustering applications.
\end{abstract}

\section{Introduction}\label{Sec:Introduction}
The Gram matrix $\bm K$, defined as ${\bm K}_{ij} = K(X_i, X_j)$ over a (sub) sample $\bm X \doteq \{X_i\}_{i=1}^n$, for a PSD kernel
$K\!:\cal X \times \cal X \to \bb R$, plays a central role in \emph{kernel machines}, where learning tasks in a (reproducing kernel) Hilbert space $\cal H$ can be performed in sample space $\cal X$ via $\bm K$. Sketching of $\bm K$, i.e., multiplying by a random matrix (or matrices) $Z\in {\bb R}^{d \times n}$---as a form of rank reduction, is now ubiquitous in the design of computationally efficient approaches to kernel machines \citep{k-means_nys, Williams-01,Yang-17}. The simplest sketching approach consists of random subsampling of columns of ${\bm K}$, i.e., a data reduction, while the usual alternative of a Gaussian sketch (of the form $Z{\bm K}$, $Z_{i, j} {\sim} \mathcal{N}(0, 1)$) has less immediate interpretation. A main aim of this paper is to derive an operator-theoretic interpretation of Gaussian sketching, i.e., understand its effect in kernel space $\cal H$ on embedded data $K(x, \cdot)$. 
The analysis reveals interesting norm preservation properties of $Z{\bm K}$, in the spirit of the Johnson-Linderstauss (J-L) lemma, even when ${\bm K}$ is viewed as a smaller submatrix of an initial gram-matrix ${\bm K}_N$ on $N\gg n$ samples; these new insights imply  an alternative use of Gaussian sketching in important applications such as kernel clustering or PCA, while yielding faster preprocessing than even vanilla Nystr\"om. 

\paragraph{Results Overview.}

It has been folklore in the community that Gaussian sketching corresponds to some form of \emph{random projection}, although it remained unclear in which formal sense this is true. To draw the link to operators on $\cal H$, we consider linear operations of the form $Z {\bm K}f_{| \bm X}\in\bb{R}^d$, where $f_{| \bm X} \doteq (f(X_1), \ldots, f(X_n))^\top$ denotes the \emph{sampled} version of $f \in \cal H$. We show that, $Z {\bm K}$, viewed in this sense as an operator, corresponds to a \emph{random} operator $\Theta$ which maps (potentially infinite-dimensional) $\cal H$ to lower-dimensional ${\bb R}^d$, while preserving---in the spirit of the J-L lemma \citep{Johnson-84, dasgupta2003elementary}---a weighted RKHS inner-product of the form $\langle f, g \rangle_{\Sigma} \doteq \langle f, \Sigma^3 g \rangle_{\cal H}$, $\Sigma$ being the \emph{covariance} operator induced by the data-generating distribution (as defined in Section~\ref{Sec:Notation}).

The corresponding random operator $\Theta$ \emph{projects}---in the informal sense of dimension reduction---any $f\in \cal H$ onto $d$ i.i.d Gaussian directions\footnote{The notion of a Gaussian measure ${\cal N}_{\cal H}$ on $\cal H$ has to be suitably defined so as to ensure that random draws $v\sim {\cal N}_{\cal H}$ are indeed elements of $\cal H$.} 
$\{v_i\}_{i =1}^d$ in $\cal H$: formally, given i.i.d. Gaussians $\{v_i\}_{i=1}^d \sim \mathcal{N}_{\cal H}(0, \Sigma^3)$, 
$\Theta$ maps any $f\in \cal H$ to the vector $\frac{1}{\sqrt{d}}(\iprod{f, v_1}_{\cal H}, \ldots, \iprod{f, v_d}_{\cal H})^\top \in {\bb R}^d.$ We refer the reader to Section~\ref{Sec:algorithm} for details.

In Section~\ref{Sec:results} (see Theorem \ref{Thm:main}), we show the following correspondence between $Z \bm K$ and $\Theta$:
just as $\Theta$ preserves 
$\iprod{g, \Sigma^3 f}_{\cal H}$, 
so does $Z \bm K$ (properly normalized), i.e., with high probability, we have $\forall f, g \in \cal H$, 
\begin{equation}\frac{1}{n^3 d}\iprod{(Z \bm K) g_{| \bm X},  (Z \bm K) f_{| \bm X}}_2 \approx \iprod{g, \Sigma^3 f}_{\cal H}\approx \iprod{\Theta g, \Theta f}_2,\end{equation}
where $\iprod{\cdot,\cdot}_2$ denotes the inner-product in ${\bb R}^d$.
The result holds simultaneously $\forall f, g \in \cal H$, with an approximation rate of order $n^{-1/2} + d^{-1/2}$, for $n, d$ greater than \emph{effective dimension} terms ($s_\Sigma$ or $s_{\Sigma^3}$ of Theorem \ref{Thm:main}). 
In other words such approximation holds for both $n$ and $d$ small, whenever the \emph{effective dimension}  is small; our experiments suggest this is often the case.

{\it Time complexity.} The above result suggests a novel use of sketching where, given a larger dataset ${\bm X}_N = \{X_i\}_{i=1}^N$, we re-map 
all $X_i \in {\bm X}_N$ to ${\bb R}^d$ using $n$ subsamples $\bm X\subset {\bm X}_N$, $n, d \ll N$, to form a 
\emph{projection} operator $Z{\bm K}$. In other words, we re-map feature functions $K(\cdot, X_i), X_i\in {\bm X}_N$ to 
$\frac{1}{n^{3/2} \sqrt{d} }(Z \bm K) K(\cdot, X_i)_{| \bm X}$, following the intuition that useful properties of 
kernel feature maps $K(\cdot, x)$ are preserved. We refer to such a mapping as Kernel JL (K-JL) for short. The time complexity is exactly $d\cdot n^2$ for forming $Z\bm K$, in addition to $N\cdot d\cdot n$ for the subsequent mapping of all $N$ datapoints. The leading constant is $1$ in all cases. In contrast, the cheapest Nystr\"om approximation using $n$ subsampled columns costs $O(d\cdot n^2)$ 
(for pseudo-inverse computation, where constants depend on desired precision) plus $N\cdot d\cdot n$ for mapping datapoints. K-JL avoids eigen-decompositions or matrix inversion steps, besides requiring smaller $n$ for stability (see Section \ref{sec:KPCA}, for details including Nystr\"om formulation).

{\it Performance.} Now, whether K-JL preserves useful properties of feature mapping depends on how the inner-product 
$\iprod{g, \Sigma^3 f}_{\cal H}$ relates to the natural inner-product $\iprod{g, f}_{\cal H}$ of the RKHS $\cal H$. In the present work, we consider clustering and PCA applications, which require that properties such as separation (in $\cal H$ distance) between given subsets of \emph{feature space} $\{K(\cdot, x): x\in \cal X\} \subset \cal H$ are preserved. At first glance, there seems to be little hope, since in the worst-case over $\cal H$, there exist $f \in \cal H$ such that 
$\|f\|^2_{\cal H} \doteq \iprod{f, f}_{\cal H}$ is large but $\iprod{f, \Sigma^3 f}_{\cal H}$ is close to $0$ (e.g., eigenfunctions $f$ of $\Sigma$ with eigenvalues tending to $0$). 

Interestingly however, as we argue in Section~\ref{sec:KPCA}, 
we can expect well-separated subsets of feature space $\{K(\cdot, x): x \in \cal X\}\subset \cal H$ to remain well-separated after K-JL, under conditions favorable to kernel PCA (KPCA) \citep{Blanchard-07,mika1999kernel,Scholkopf-98}, or conditions favorable to kernel $k$-means (K-$k$-means) \citep{Dhillon-04}. Namely, if feature maps $K(\cdot, x)$ lie close to a low-dimensional subspace, or feature maps \emph{cluster} well (in which case the means of clusters lie close to a low-dimensional subspace), then the worst-case distortions between the two inner-products happen outside of feature space 
$\{K(\cdot, x): x \in \cal X\}$. This entails similar benefits as in KPCA and or K-$k$-means, albeit at the cheaper cost of a random projection. This intuition holds empirically, as we verify on a mix of synthetic data and real-world clustering applications. 

\paragraph{Further Related Work.} We note that \emph{sketching} is of general interest outside the present context, motivated by the need for efficient approximations of general matrices appearing in numerical and data analysis \citep{woodruff2014sketching, andoni2016sketching, andoni2018sketching}. Finally, we note that the benefits of Johnson-Linderstrauss type projections in Hilbert spaces were considered in \cite{biau2008performance}, however under the assumptions of a theoretical procedure which requires explicit Fourier coefficients (basis expansion) of Hilbert space elements $K(X_i, \cdot)$. 

\subsection*{Paper Outline} 

Section~\ref{Sec:Notation} covers definitions and basic assumptions used throughout the paper. In Section \ref{Sec:algorithm} we develop some initial intuition about JL-type \emph{random projections} in $\cal H$, followed by formal results in Section \ref{Sec:results}. Omitted proofs and supporting results are 
collected in an appendix.

\section{Preliminaries}\label{Sec:Notation}
Let $\Cal{X}$ denote a separable topological space on which a Borel probability measure $\rho_X$ is defined. We assume that $\Cal{H}$, consisting of functions $\cal X \to \bb R$, is a reproducing kernel Hilbert space (RKHS) with a continuous and bounded reproducing kernel $K:\Cal{X}\times\Cal{X}\to\bb{R}$ where $\sup_{x\in\Cal{X}}K(x,x)=:\kappa<\infty$. For any $f \in \cal H$, the \emph{outer-product} notation $f \otimes_\Cal{H} f$ denotes the operator $g\mapsto \iprod{g, f}_{\cal H} f$. We let $\Sigma:\Cal{H}\rightarrow\Cal{H}$ denote the uncentered covariance operator, which is defined as 

$$\Sigma \doteq \int K(\cdot,x){\otimes_\Cal{H}} K(\cdot,x)\,d\rho_X(x),$$ 

in the sense of Bochner integration \citep{Diestel-77}.
Given data $\{X_i\}^n_{i=1}\stackrel{i.i.d.}{\sim}\rho_X$ where $n\ge 1$, the empirical counterpart of $\Sigma$ is defined as 
$$\Sigma_n\doteq\frac{1}{n}\sum^n_{i=1}K(\cdot,X_i)\otimes_\Cal{H} K(\cdot,X_i).$$

Given two normed spaces $(\Cal{F},\Vert\cdot\Vert_\Cal{F})$ and $(\Cal{G},\Vert\cdot\Vert_\Cal{G})$, let $A:\Cal{F}\rightarrow\Cal{G}$ and $B:\Cal{F}\rightarrow\Cal{F}$ be two linear operators. The operator norm  of $A$ is defined as $\Vert A\Vert_{\text{op}} \doteq \sup_{f\in\Cal{F}}\frac{\Vert Af\Vert_\Cal{G}}{\Vert f\Vert_\Cal{F}}.$

The trace of a non-negative self-adjoint operator $B$, operating on a separable Hilbert space $\Cal{F}$, is defined as $\text{tr}(B) \doteq \sum_\ell \langle Be_\ell,e_\ell\rangle_\Cal{F},$
where $(e_\ell)_\ell$ is any orthonormal basis in $\Cal{F}$.  The Hilbert-Schmidt norm of $B$ is then defined as $\Vert B\Vert_{\Cal{L}_2(\Cal{F})} \doteq \sqrt{\text{tr}(B^*B)}$.

 A random element $v$ of $\cal H$
is said to have {\bf Gaussian} measure, denoted $\mathcal{N}_{\mathcal{H}}$, if for any $f \in \cal H$, $\iprod{v, f}_{\cal H}$ is Gaussian. It is known that  such a measure is well-defined, in the sense that $v\sim \mathcal{N}_{\mathcal{H}}$ has finite norm 
$\|v\|_{\cal H}$ w.p. 1, whenever its corresponding \emph{covariance operator} 
${\cal C} \doteq \bb E \, v\otimes_{\cal H} v - \mu\otimes_{\cal H}\mu, \, \mu = \bb E v$, is trace-class, i.e., has finite trace (see e.g., \citealp{bogachev98gaussian}). We can then parametrize the measure as  
$\mathcal{N}_{\mathcal{H}}(\mu,{\cal C})$.
\section{Intuition on Random Projection in RKHS} \label{Sec:algorithm}
As mentioned in Section~\ref{Sec:Introduction}, a key contribution of this paper is in showing that the random projection operator $\Theta$ is related to the Gaussian sketch of a kernel matrix. Before we present and prove a rigorous result in Section~\ref{Sec:results}, in this section, we heuristically demonstrate the connection. In particular, we elucidate why $\Sigma^3$ shows up (rather than e.g. $\Sigma$, given that a priori $\bm K$ seems most naturally related to $\Sigma$), and why the origin of the peculiar normalization by $n^{3/2}\sqrt{d}$. 

Given a set of $N$ datapoints in $\bb{R}^D$, classical random projections in the style of Johnsohn-Lindenstrauss (J-L) consists of projecting the datapoints onto $d$ random directions which are sampled from a standard Gaussian distribution. The same idea can be intuitively carried forward to an RKHS, $\Cal{H}$ by sampling functions from a Gaussian measure on $\Cal{H}$---these functions act as directions along which a function in $\Cal{H}$ can be projected. Now, consider the random directions $v_i\stackrel{i.i.d.}{\sim}\mathcal{N}_{\mathcal{H}}(0,\Sigma^3)$ and define the random projection of $f\in\Cal{H}$ to $\bb{R}^d$ through the random operator $\Theta:\Cal{H}\rightarrow\bb{R}^d$
\begin{align}
f\mapsto \frac{1}{\sqrt{d}}\left(\langle v_1,f\rangle_\Cal{H},\ldots, \langle v_d,f\rangle_\Cal{H}\right)^\top. \label{eq:Theta}
\end{align}
It is important to note that random directions cannot be sampled from a Gaussian measure with identity covariance operator $I_{\cal H}$ (i.e., similar to the classical setting) as such a measure is not well-defined for infinite dimensional Hilbert spaces 
since $I_{\cal H}$ has infinite trace. The above normalization by $d^{-1/2}$ ensures that, with high-probability, 
$\iprod{\Theta g, \Theta f}_{2} \xrightarrow{d\to \infty}\iprod{g, \Sigma^3 f}_{\cal H}$ 
{(see Proposition \ref{prop:Thetaconvergence})}.

Now define $(u_i)^d_{i=1}\stackrel{i.i.d.}{\sim}\Cal{N}_\Cal{H}(0,\Sigma)$ so that $(v_i)^d_{i=1}$ can be written as $v_i=\Sigma u_i$ and 
\begin{equation}\langle v_i,f\rangle_\Cal{H}=\langle \Sigma u_i,f\rangle_\Cal{H}=\int f(x) u_i(x)\,d\rho_X(x)
=\langle u_i,f\rangle_{L^2(\Cal{X},\rho_X)}.\nonumber\end{equation} 
The above can be approximated empirically, using $\bm{X} \doteq \{X_i\}^n_{i=1}\stackrel{i.i.d.}{\sim}\rho_X$, as 
\begin{equation}\langle u_i,f\rangle_{L^2(\rho_X)}\approx\frac{1}{n}\sum^n_{j=1}f(X_j)u_i(X_j)=\frac{1}{n}\langle \sx u_i,\sx f\rangle_2,\end{equation}
where $$\sx:\Cal{H}\rightarrow\bb{R}^n,\,\,f\mapsto (f(X_1),\ldots,f(X_n))^\top$$ is a \emph{sampling operator} \citep{Smale-07} whose adjoint is given by $$\sxs:\bb{R}^n\rightarrow\Cal{H},\,\,\bm{\beta}\mapsto\sum^n_{i=1}\beta_i K(\cdot,X_i).$$ 
It follows from Proposition~\ref{pro:gauss} (in the appendix) that, conditioned on the sample $\bm{X}$, $\sx u_{i}$ is distributed as $\Cal{N}(0,M)$ where $M \in {\bb R}^{n\times n}$ is defined as 
\begin{equation}
M_{jl}
\doteq \langle K(\cdot,X_j),\Sigma K(\cdot,X_l)\rangle_\Cal{H}=\int_\Cal{X} K(x,X_j) K(x,X_l)\,d\rho_X(x), 
\label{Eq:M}
\end{equation}
with $X_j, X_l \in \bm X.$
Based on $\bm{X}$, $M$ can be further approximated as $\hat M$ where 
\begin{equation}
\hat{M}_{jl} \doteq \langle K(\cdot,X_j),\Sigma_n K(\cdot,X_l)\rangle_\Cal{H}=\frac{1}{n}\sum^n_{i=1}K(X_i,X_j)K(X_i,X_l)=\frac{1}{n}(\bm{K}^2)_{jl},
\label{Eq:hatM}
\end{equation}
where $\bm{K}$ is the Gram matrix based on $\bm{X}$. To summarize, we have carried out the following sequence of approximations to $\langle v_i,f\rangle_\Cal{H}$:
$$\langle v_i,f\rangle_\Cal{H}=\langle u_i,f\rangle_{L^2(\Cal{X},\rho_X)}\approx \frac{1}{n}\langle \sx u_i,\sx f\rangle_2$$ where $\sx u_i\sim \Cal{N}(0,M)\approx \Cal{N}\left(0,\frac{1}{n}\bm{K}^2\right).$
This means an approximation to $\langle v_i,f\rangle_\Cal{H}$ can be obtained by sampling, say $\hat{v}_i$ from $\Cal{N}\left(0,\frac{1}{n}\bm{K}^2\right)$ and computing $\frac{1}{n}\langle \hat{v}_i,\sx f\rangle_2$. Recalling the form of $\Theta$ \eqref{eq:Theta}, define 
$$\hat{V}=\frac{1}{n\sqrt{d}}[\hat{v}_1,\ldots,\hat{v}_d]\in \bb{R}^{n\times d}.$$ The \emph{approximate random projection operator} is then 
$\hat{V}^\top\sx :\Cal{H}\rightarrow \bb{R}^d$, where 
$$f\mapsto \hat{V}^\top\sx f=\frac{1}{n\sqrt{d}}(\langle \hat{v}_1,\sx f \rangle_2,\ldots, \langle \hat{v}_d,\sx f \rangle_2)^\top.$$
Note that $\hat{V}^\top=\frac{1}{n\sqrt{d}}[\hat{v}_1,\ldots,\hat{v}_d]^\top=\frac{1}{n\sqrt{nd}}Z\bm{K}$ with $Z\in\bb{R}^{d\times n}$ having i.i.d.~$\Cal{N}(0,1)$ entries.
\section{Main Results}\label{Sec:results}
In this section, we formalize the relation between $\Theta$ and $\hat V^\top S_{\bm X}$ by showing that, with high-probability, 
$\iprod{ \hat V^\top g_{| \bm X},  \hat V^\top  f_{| \bm X}}_2 \approx \iprod{g, \Sigma^3 f}_{\cal H} \approx \iprod{\Theta g, \Theta f}_2$. This relation is established in Proposition\ref{prop:Thetaconvergence} (proved in Section~\ref{pro:sigma3}), and Theorem \ref{Thm:main}. In the sequel, we let 
$a \land b \doteq \min \{a, b\}$ and $a \lor b \doteq \max \{a, b\}$. 

\begin{pro} 
\label{prop:Thetaconvergence}
Define $s_\Sigma=\frac{\emph{tr}(\Sigma)}{\Vert \Sigma\Vert_{\emph{op}}}$. For any $\tau\ge 1 $ and $d\ge (s_{\Sigma^3}\vee \tau)$, with probability at least $1-e^{-\tau}$,
\begin{eqnarray}
\sup_{f,g\in\Cal{H}}\frac{\left|\langle \Theta g,\Theta f\rangle_2-\langle g,\Sigma^3 f\rangle_\Cal{H}\right|}{\Vert f\Vert_\Cal{H}\Vert g\Vert_\Cal{H}}\le 
\mathfrak{C}\Vert \Sigma\Vert^3_{\emph{op}}\frac{\sqrt{s_{\Sigma^3}}+\sqrt{\tau}}{\sqrt{d}},\nonumber
\end{eqnarray}
where $\mathfrak{C}$ is a universal constant independent of $\Sigma$, $\tau$ and $d$.
\end{pro}


\begin{thm}[Convergence of inner products]\label{Thm:main}
Let 
$\tau\ge 1$. Define $s_\Sigma=\frac{\emph{tr}(\Sigma)}{\Vert\Sigma\Vert_{\emph{op}}}$. 
Suppose 
$$n\ge \frac{6272\kappa s^5_\Sigma\tau}{\Vert \Sigma\Vert_{\emph{op}}}
.$$
%
Then, with probability at least $1-5e^{-\tau}$ jointly over the choice of 
$\{\hat{v}_i\}^d_{i=1}$ and $\{X_i\}^n_{i=1}$:
\begin{eqnarray}
\sup_{f,g\in\Cal{H}}\frac{\left|\left\langle \hat{V}^\top \sx g,\hat{V}^\top  \sx f\right\rangle_2-\left\langle g,\Sigma^3f \right\rangle_\Cal{H}\right|}{\Vert f\Vert_\Cal{H}\Vert g\Vert_\Cal{H}}
&{}\le{}& 3\mathfrak{C}\Vert \Sigma\Vert^3_{\emph{op}}\frac{\sqrt{2s_{\Sigma^3}}+\sqrt{\tau}}{\sqrt{d}} 
+\frac{28 \Vert \Sigma\Vert^{5/2}_{\emph{op}}\sqrt{2\kappa s_\Sigma\tau}}{\sqrt{n}},\nonumber
\end{eqnarray}
where $\frak{C}$ is a universal constant that does not depend on $n$, $d$, $\kappa$ and $\Sigma$.
\end{thm}
{\bf Remark.} (Main dependence on $n$ and $d$) The leading constants above are in terms of $\|\Sigma\|_{\text{op}} \leq \kappa \doteq \sup_x K(x, x)$, and are therefore expected to be small for common kernels such as Gaussian ($\kappa = 1$). Thus the main dependence on $n$ and $d$ in the rates are given by $s_\Sigma$ and $s_{\Sigma^3} \leq s_\Sigma$, viewed as \emph{effective dimension} terms. Thus, whenever $s_\Sigma$ is small, both $n$ and $d$ can be chosen small while maintaining the guarantees of the above theorem. In our experiments of Section \ref{sec:KPCA}, $d \leq n \leq 100$ is often sufficient even for datasizes in excess of 40K points (see e.g. Figure \ref{fig:varyingND}).
\begin{proof}
Define
$A(f,g) \doteq \left\langle  \hat{V}^\top\sx g,\hat{V}^\top\sx  f\right\rangle_2-\left\langle g,\Sigma^3f \right\rangle_\Cal{H}$. Since $$\left\langle  \hat{V}^\top\sx g,\hat{V}^\top\sx  f\right\rangle_2=\left\langle  g,\sxs\hat{V}\hat{V}^\top\sx  f\right\rangle_\Cal{H},$$ we have 
$\sup_{f,g\in\Cal{H}}\frac{|A(f,g)|}{\Vert f\Vert_\Cal{H}\Vert g\Vert_\Cal{H}}=
\sup_{f\in\Cal{H}}\frac{\left\Vert \left(\sxs\hat{V}\hat{V}^\top\sx-\Sigma^3\right)  f\right\Vert_\Cal{H}}{\Vert f\Vert_\Cal{H}}
=\left\Vert \sxs\hat{V}\hat{V}^\top\sx-\Sigma^3\right\Vert_{\text{op}}.$ 
In the following, we bound $\left\Vert \sxs\hat{V}\hat{V}^\top\sx-\Sigma^3\right\Vert_{\text{op}}$. To this end, consider
\begin{eqnarray}
\sxs\hat{V}\hat{V}^\top\sx-\Sigma^3&{}={}& \sxs\left(\hat{V}\hat{V}^\top-\frac{1}{n^3}\bm{K}^2\right)\sx +\frac{1}{n^3}\sxs \bm{K}^2 \sx-\Sigma^3\nonumber\\
&{}\stackrel{(\dagger)}{=}{}&
\sxs\left(\hat{V}\hat{V}^\top-\frac{1}{n^3}\bm{K}^2\right)\sx+\Sigma^3_n-\Sigma^3,\nonumber
\end{eqnarray}
where in $(\dagger)$, we use the facts that $\bm{K}=\sx\sxs$ and $\Sigma_n=\frac{1}{n}\sxs\sx$. Therefore,
\begin{eqnarray}
\left\Vert \sxs\hat{V}\hat{V}^\top\sx-\Sigma^3\right\Vert_{\text{op}}&{}\le{}& \left\Vert \sxs\left(\hat{V}\hat{V}^\top-\frac{1}{n^3}\bm{K}^2\right)\sx\right\Vert_{\text{op}}+\Vert \Sigma^3_n-\Sigma^3\Vert_{\text{op}}\nonumber\\
&{}\stackrel{(\ddagger)}{\le}&{}\frak{C} \left(\sqrt{\frac{\text{tr}(\Sigma^3_n)\Vert \Sigma^3_n\Vert_{\text{op}}}{d}}+ \Vert \Sigma^3_n\Vert_{\text{op}}\sqrt{\frac{\tau}{d}}\right)+\Vert \Sigma^3_n-\Sigma^3\Vert_{\text{op}},\nonumber\\
&{}\le{}&\frak{C} \left(\sqrt{\frac{\text{tr}(\Sigma^3_n)\Vert \Sigma^3_n-\Sigma^3\Vert_{\text{op}}}{d}}+ \sqrt{\frac{\text{tr}(\Sigma^3_n)\Vert\Sigma^3\Vert_{\text{op}}}{d}}+\Vert \Sigma^3\Vert_{\text{op}}\sqrt{\frac{\tau}{d}}\right)\nonumber\\
&{}{}&\qquad\qquad+\Vert \Sigma^3_n-\Sigma^3\Vert_{\text{op}}\left(1+\frak{C}\sqrt{\frac{\tau}{d}}\right),\label{Eq:main}
\end{eqnarray}
where $(\ddagger)$ follows from Lemma~\ref{lem:hatU}, which holds with probability $1-e^{-\tau}$ over the choice of $(\hat{v}_i)^d_{i=1}$ conditioned on $\bm{X}$ for any $\tau\ge 1$ and $d\ge \left(\frac{\text{tr}(\Sigma^3_n)}{\Vert \Sigma^3\Vert_{\text{op}}-\Vert \Sigma^3-\Sigma^3_n\Vert_{\text{op}}}\vee \tau\right)\ge  \left(\frac{\text{tr}(\Sigma^3_n)}{\Vert \Sigma^3_n\Vert_{\text{op}}}\vee \tau\right)$ with $\frak{C}$ being a universal constant independent of $\sxs \bm{K}^2\sx$ and $d$. We now bound $\text{tr}(\Sigma^3_n)$ and $\Vert \Sigma^3_n-\Sigma^3\Vert_{\text{op}}$. Consider
\begin{eqnarray}
\Sigma^3_n-\Sigma^3&{}={}&(\Sigma_n-\Sigma+\Sigma)^3-\Sigma^3\nonumber\\
&{}={}&(\Sigma_n-\Sigma)^3+(\Sigma_n-\Sigma)^2\Sigma+(\Sigma_n-\Sigma)\Sigma(\Sigma_n-\Sigma)+(\Sigma_n-\Sigma)\Sigma^2\nonumber\\
&{}{}&\qquad+\Sigma(\Sigma_n-\Sigma)^2+\Sigma(\Sigma_n-\Sigma)\Sigma+\Sigma^2(\Sigma_n-\Sigma),\nonumber
\end{eqnarray}
which yields
\begin{eqnarray}
\Vert \Sigma^3_n-\Sigma^3\Vert_{\text{op}}&{}\le{}&\Vert\Sigma_n-\Sigma\Vert^3_{\text{op}}+3\Vert\Sigma_n-\Sigma\Vert^2_{\text{op}}\Vert\Sigma\Vert_{\text{op}}+3\Vert\Sigma_n-\Sigma\Vert_{\text{op}}\Vert\Sigma\Vert^2_{\text{op}}\nonumber\\
&{}\le{}&\Vert\Sigma_n-\Sigma\Vert^3_{\hs}+3\Vert\Sigma_n-\Sigma\Vert^2_{\hs}\Vert\Sigma\Vert_{\text{op}}+3\Vert\Sigma_n-\Sigma\Vert_{\hs}\Vert\Sigma\Vert^2_{\text{op}}\nonumber
\end{eqnarray}
and
\begin{eqnarray}
\text{tr}(\Sigma^3_n)&{}\le{}& \text{tr}(\Sigma^3)+\Vert\Sigma_n-\Sigma\Vert^3_{\hs}+3\Vert \Sigma\Vert_{\hs}\Vert\Sigma_n-\Sigma\Vert^2_{\hs}
+3\Vert \Sigma\Vert^2_{\hs}\Vert\Sigma_n-\Sigma\Vert_{\hs}.\nonumber
\end{eqnarray}
It follows from Lemma~\ref{lem:bern} that for any $\tau>0$ 
and 
$n\ge \frac{32\kappa s_\Sigma \tau}{\Vert \Sigma\Vert_{\text{op}}}$, we have 
\begin{equation}\Vert \Sigma^3_n-\Sigma^3_n\Vert_{\text{op}}\le 
28 \Vert \Sigma\Vert^2_{\text{op}}\sqrt{\frac{2\kappa\text{tr}(\Sigma)\tau}{n}}
\label{Eq:sigma3}\end{equation}
and
\begin{equation}\text{tr}(\Sigma^3_n)\le \text{tr}(\Sigma^3)+28 \Vert \Sigma\Vert^2_{\hs}\sqrt{\frac{2\kappa\text{tr}( \Sigma)\tau}{n}},\label{Eq:sigmatr}\end{equation}
where each of the above inequalities hold with probability at least $1-2e^{-\tau}$ over $(X_1,\ldots,X_n)$. 
Using \eqref{Eq:sigma3} and \eqref{Eq:sigmatr} in \eqref{Eq:main} yields the result, upon typing a few loose ends. 

Define $\Delta \doteq 28 \Vert \Sigma\Vert^2_{\hs}\sqrt{\frac{2\kappa\text{tr}( \Sigma)\tau}{n}}$ and $\Delta^\prime\doteq \frac{\Vert\Sigma\Vert^2_{\text{op}}}{\Vert\Sigma\Vert^2_{\hs}}\theta$. As aforementioned, \eqref{Eq:main} holds if $d\ge \left(\frac{\text{tr}(\Sigma^3_n)}{\Vert\Sigma^3\Vert_{\text{op}}-\Vert\Sigma^3_n-\Sigma^3\Vert_{\text{op}}}\vee \tau\right)$ which is true if $d\ge \left(\frac{\text{tr}(\Sigma^3)+\Delta}{\Vert\Sigma^3\Vert_{\text{op}}-\Delta^\prime}\vee\tau\right)$. Under the assumed conditions on $n$, it follows that $\Vert\Sigma\Vert^3_{\text{op}}\ge 2\Delta^\prime$ and $\text{tr}(\Sigma^3)\le \Delta$, which yields that $d\ge \left(\frac{\text{tr}(\Sigma^3)+\Delta}{\Vert\Sigma^3\Vert_{\text{op}}-\Delta^\prime}\vee\tau\right)$ is true if $d\ge (4s_{\Sigma^3}\vee \tau).$
%
\end{proof}
Theorem~\ref{Thm:main} shows that the approximate random projection operator $\hat{V}^\top \sx$ preserves the inner product $\langle g,\Sigma^3 f\rangle_\Cal{H}$ uniformly over all $f,g\in\Cal{H}$ at an approximation rate of $n^{-1/2}+d^{-1/2}$. 

The following result provides a different angle by which $\Theta$ relates to $\hat V^\top S_{\bm X}$, by showing that, for all $\bm{\alpha}\in\bb{R}^d$ and $f\in\Cal{H}$, 
 $\langle \bm{\alpha},\hat{V}^\top\sx f\rangle_2$ converges in probability to $\langle \bm{\alpha},\Theta f\rangle_2$ at the rate of $d^{-1/2}$, provided $n$ is large enough for $\|\Sigma_n\|_{\text{op}}\lesssim \|\Sigma\|_{\text{op}}$. Recall that two operators $A, B: {\cal H} \to {\bb R}^d$ are equal (in a \emph{weak sense}) if $\forall f \in {\cal H},\, \forall 
 {\bm \alpha} \in {\bb R}^d$, we have $\iprod{\bm \alpha, Af}_2  = \iprod{\bm \alpha, Bf}_2$. 
\begin{thm}[Convergence of random projection operators]\label{thm:proj}
Define $s_\Sigma=\frac{\emph{tr}(\Sigma)}{\Vert\Sigma\Vert_{\emph{op}}}$. For any $\bm{\alpha}\in\bb{R}^d$, $f\in\Cal{H}$, $\tau>0$ and $$n\ge \kappa s_\Sigma\tau\left(\frac{32}{\Vert\Sigma\Vert_{\emph{op}}}\vee 1\right),$$
with probability at least $1-4e^{-\tau}$ jointly over the choice of 
$\{\hat{v}_i\}^d_{i=1}$, $\{v_i\}^d_{i=1}$ and $\{X_i\}^n_{i=1}$:
$$\left|\langle \bm{\alpha},\hat{V}^\top\sx f\rangle_2-\langle \bm{\alpha},\Theta f\rangle_2\right|\le \frac{16\sqrt{2\tau}\Vert \bm{\alpha}\Vert_2\Vert f\Vert_\Cal{H}\left(\Vert \Sigma\Vert^{3/2}_{\emph{op}}\vee \Vert \Sigma\Vert^{5/4}_{\emph{op}}\right)}{\sqrt{d}}
.$$
\end{thm}
\begin{proof}
Note that $\langle \bm{\alpha},\hat{V}^\top\sx f\rangle_2=\langle \hat{V}\bm{\alpha},\sx f\rangle_2=\frac{1}{n\sqrt{d}}\left\langle\sum^d_{i=1}\alpha_i\hat{v}_i,\sx f\right\rangle_2$. Since $(\hat{v}_i)\stackrel{i.i.d.}{\sim}\Cal{N}(0,\frac{1}{n}\bm{K}^2)$, conditioned on $\bm{X}$ it follows that $$\frac{1}{n\sqrt{d}}\left\langle\sum^d_{i=1}\alpha_i\hat{v}_i,\sx f\right\rangle_2\sim\Cal{N}\left(0,\frac{1}{n^3d}\Vert\bm{\alpha}\Vert^2_2\langle \sx f,\bm{K}^2\sx f\rangle_2\right)=\Cal{N}\left(0,\frac{1}{d}\Vert\bm{\alpha}\Vert^2_2\langle f,\Sigma^3_n f\rangle_\Cal{H}\right),$$
where we used $\bm{K}=\sx\sxs$ and $n\Sigma_n=\sxs\sx$. On the other hand, $\langle\bm{\alpha},\Theta f\rangle_2=\frac{1}{\sqrt{d}}\sum^d_{i=1}\alpha_i\langle v_i,f\rangle_\Cal{H}\sim \Cal{N}\left(0,\frac{1}{d}\Vert\bm{\alpha}\Vert^2_2\langle f,\Sigma^3 f\rangle_\Cal{H}\right)$ which follows from the fact that $(v_i)^d_{i=1}\stackrel{i.i.d.}{\sim}\Cal{N}_{\Cal{H}}(0,\Sigma^3)$ which in turn implies $\langle v_i,f\rangle_\Cal{H}\sim \Cal{N}(0,\langle f,\Sigma^3 f\rangle_\Cal{H})$. Therefore $\langle \bm{\alpha},\hat{V}^\top\sx f-\Theta f\rangle_2\sim \Cal{N}\left(0,\frac{1}{d}\Vert\bm{\alpha}\Vert^2_2\langle f,(\Sigma^3+\Sigma^3_n)f\rangle_\Cal{H}\right)$ 
conditioned on $\bm{X}$. For $Y\sim \Cal{N}(0,\sigma^2)$, the Gaussian concentration inequality yields that for any $\tau>0$, with probability at least $1-2e^{-\tau}$, $|Y|\le \sqrt{2\sigma^2\tau}$. Hence it follows that for any $\tau>0$, 
with probability at least $1-2e^{-\tau}$ jointly over $\{v_i\}^d_{i=1}$, $\{\hat{v}_i\}^d_{i=1}$ and conditioned on $\bm{X}$, we obtain
\begin{eqnarray}\left|\langle \bm{\alpha},\hat{V}^\top\sx f-\Theta f\rangle_2\right|&{}\le{}& \Vert\bm{\alpha}\Vert_2\sqrt{\frac{2\tau\langle f,(\Sigma^3+\Sigma^3_n)f\rangle_\Cal{H}}{d}}\nonumber\\
&{}\le{}& \Vert\bm{\alpha}\Vert_2\left(\frac{2\Vert\Sigma\Vert^{3/2}_{\text{op}}\Vert f\Vert_\Cal{H}\sqrt{\tau}}{\sqrt{d}}+\frac{\Vert\Sigma^3-\Sigma^3_n\Vert^{1/2}_{\text{op}}\Vert f\Vert_\Cal{H}\sqrt{2\tau}}{\sqrt{d}}\right),
\label{Eq:proj}
\end{eqnarray} 
which follows from the fact that
$\langle f,(\Sigma^3+\Sigma^3_n)f\rangle_\Cal{H}=2\langle f,\Sigma^3 f\rangle_\Cal{H}+\langle f,(\Sigma^3_n-\Sigma^3)f\rangle_\Cal{H}\le 2\Vert \Sigma\Vert^3_{\text{op}}\Vert f\Vert^2_\Cal{H}+\Vert f\Vert^2\Vert \Sigma^3_n-\Sigma^3\Vert_{\text{op}}$. Therefore, by unconditioning w.r.t.~$\bm{X}$, \eqref{Eq:proj} holds with probability at least $1-2e^{-\tau}$ jointly over the choice of 
$\{\hat{v}_i\}^d_{i=1}$, $\{v_i\}^d_{i=1}$ and $\{X_i\}^n_{i=1}$. The result therefore follows by invoking \eqref{Eq:sigma3} to bound $\Vert \Sigma^3-\Sigma^3_n\Vert_{\text{op}}$ in \eqref{Eq:proj}. 
\end{proof}

%

In Theorem~\ref{thm:proj}, we require $d\rightarrow\infty$ to achieve $\langle \bm{\alpha},\hat{V}^\top\sx f\rangle_2\rightarrow \langle \bm{\alpha},\Theta f\rangle_2$ in probability for all $\bm{\alpha}\in\bb{R}^d$ and $f\in\Cal{H}$, for sufficiently large $n$. Instead in the following result, we keep $d$ fixed, and show the convergence in distribution of $\hat{V}^\top\sx f$ 
to $\Theta f$ as $n\rightarrow\infty$ for all $f\in\Cal{H}$ as $n \to \infty$. 

\begin{thm}\label{prop:convergenceinD}
For all $f \in \cal H$ we have that  
$$\hat V^\top S_{\bm X} f \xrightarrow{\text{dist}} \Theta f, \text{ as } n \to \infty.$$
\end{thm}

\section{K-JL relations to KPCA and K-$k$-means}
\label{sec:KPCA}
In the next two subsections we argue, through simple corollaries to Theorem \ref{Thm:main} and experimental evaluation, that 
K-JL preserves geometric and clustering aspects of kernel PCA (KPCA) \citep{mika1999kernel,Scholkopf-98} and kernel $k$-means (K-$k$-means) \citep{Dhillon-04}, at the cheaper costs of random projection, whenever favorable conditions for KPCA, resp. K-$k$-means hold in practice. 

Recall that, given a large dataset ${\bm X}_N$ of size $N$, K-JL consists of remapping each $X_i \in {\bm X}_N$ as 
$\hat V^\top S_{\bm X} K(\cdot, X_i)$, where $\bm X$ is the size $n$ subsample of ${\bm X}_N$ used to compute 
$\hat V^\top\doteq \frac{1}{n\sqrt{nd}}Z\bm{K}$. 

\subsection{Preserving Low-dimensional Separation}
In this section, we develop the intuition that, Kernel JL has similar advantages as Kernel PCA (KPCA)
under situations favorable to KPCA. In particular, KPCA works under the assumption that the data in feature space 
$\{K(\cdot, x): x \in \cal X\}$ lies close to a lower-dimensional eigenspace of the covariance operator 
$\Sigma$ \citep{Blanchard-07,mika1999kernel,Scholkopf-98}. We formalize this assumption below.

\begin{assumption}[KPCA]
\label{assumption:KPCA}
Let $\Sigma = \sum_i \lambda_i (f_i  \otimes_{\cal H} f_i)$ denote a spectral decomposition of $\Sigma$ (with non-increasing eigenvalues $\{\lambda_i\}$, and assume 
$\bb E_{X\sim\rho_X} K(\cdot, X) = 0$.  For any $k \in \bb N$, let $P_k$ denote the projection operator onto $\text{span}\{f_i: i \in [k]\}$. There exists $k \in \bb N$, and $0<\epsilon, \eta<1$ such that 
$$ \rho_X\left \{x: \|P_k K(\cdot, x) \|_{\cal H}^2 \geq (1-\epsilon) \| K(\cdot, x)\|_{\cal H}^2\right\} \geq 1-\eta.$$
\end{assumption}

We start with some theoretical intuition using the following formal example.
\begin{example}[Well-separated subsets of feature space] \label{ex:1}
Let Assumption \ref{assumption:KPCA} hold for some $0<\epsilon, \eta < 1$. Let $\tau > 1$. 
Let ${\bm X}_N$ denote an i.i.d. sample of size $N$ from $\rho_X$ (not necessarily independent from 
$\bm X$, since the results of Theorem \ref{Thm:main} hold uniformly over $\cal H$). 

The following holds with probability at least $1- e^{-\tau} - N\eta$, over any 
subsets ${\cal F}, {\cal G}$ of $\{K(\cdot, x): x\in {\bm X}_N \}$ satisfying  
$\min_{f \in {\cal F}, \, g \in {\cal G}} \| f- g\|^2_{\cal H} = \Delta$, for some separation $\Delta = \Delta({\cal F}, {\cal G}) > 0$. We have for some $C_1, C_2$,  both functions of $(K, \rho_X)$, that for $n\land d > C_1$: 
\begin{equation} 
\inf_{f \in {\cal F}, \, g \in {\cal G}} \left \|{\hat V}^\top\sx \left(f-g\right) \right\|^2
\geq \lambda_k^3\cdot (\Delta - 2\epsilon\cdot \kappa)- C_2\sqrt{\frac{\tau}{n\land d}}. \label{eq:pca-lbound}
\end{equation}
On the other hand, independent of Assumption \ref{assumption:KPCA}, we have with probability at least $1-e^{-\tau}$ that for all $f, g \in \cal H$ we have the upper-bound
\begin{align}
\left \|{\hat V}^\top\sx \left(f-g\right) \right\|^2 \leq \lambda_1^3 \cdot \|f- g\|^2_{\cal H} + C_2\sqrt{\frac{\tau}{n\land d}}. \label{eq:pca-ubound}
\end{align}
\end{example}
The above is obtained by noticing that, we have for any $h \in \cal H$ that 
\begin{align*} 
\iprod{h, \Sigma^3 h}_\Cal{H} &= \sum_{i=1}^\infty \lambda_i^3 \iprod{h, f_i}^2
\geq \lambda_k^3 \cdot \sum_{i = 1}^k \iprod{h, f_i}^2 = \lambda_k^3 \cdot \left \| P_k h \right\|_{\cal H}^2, \text{ and similarly }  \\
 \iprod{h, \Sigma^3 h}_\Cal{H} &\leq  \lambda_1^3 \cdot \sum_{i=1}^\infty \iprod{h, f_i}^2 = \lambda_1^3 \cdot \left \| h \right\|^2_{\cal H}. 
\end{align*} 
Now, under Assumption \ref{assumption:KPCA} and Theorem \ref{Thm:main}, take $h \doteq f -g$ to obtain the statements of \eqref{eq:pca-lbound} and \eqref{eq:pca-ubound}. For \eqref{eq:pca-lbound}, notice further that, under Assumption \ref{assumption:KPCA}, we have with probability at least $1-N\eta$ that 
\begin{align*}
\pushQED{\qed}     
\|P_k (f -g)\|_{\cal H}^2 &= \|f -g\|_{\cal H}^2 - \|P_k^\perp (f -g)\|_{\cal H}^2
\geq \|f -g\|_{\cal H}^2 -2 \left ( \|P_k^\perp f\|_{\cal H}^2  + \|P_k^\perp g\|_{\cal H}^2 \right ) \\
&\geq \|f -g\|_{\cal H}^2 - 2\epsilon\cdot \kappa. \qedhere
\popQED
\end{align*}

From the above, if the two subsets $\cal F, \cal G$ are well-separated in feature space under KPCA, they remain well-separated after K-JL, provided the \emph{condition number} $\lambda_1/\lambda_k$ is not too large: distances are rescaled below by 
$\lambda_k^3$, but rescaled above by $\lambda_1^3$. In the favorable 
case  where $\lambda_1/\lambda_k \approx 1$, we see from \eqref{eq:pca-lbound} and \eqref{eq:pca-ubound} that K-JL should achieve similar separation properties as KPCA, provided $\Delta$ is large w.r.t. to interpoint distances in $\cal F$ and $\cal G$. 
This intuition is formalized in the following example where we consider a scale-free notion of separation. 

\begin{figure}[tb]
\setlength\abovecaptionskip{-0.7\baselineskip}
\centering 
\begin{minipage}[b]{0.21\textwidth}
\includegraphics[width=\textwidth]{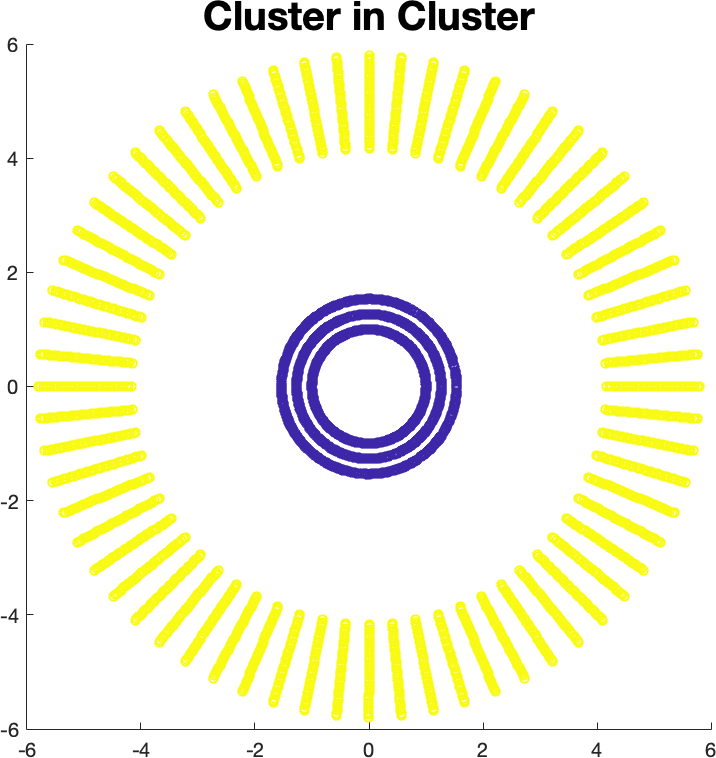}
\end{minipage}
\hspace{3mm}
\begin{minipage}[b]{0.21\textwidth}
\includegraphics[width=\textwidth]{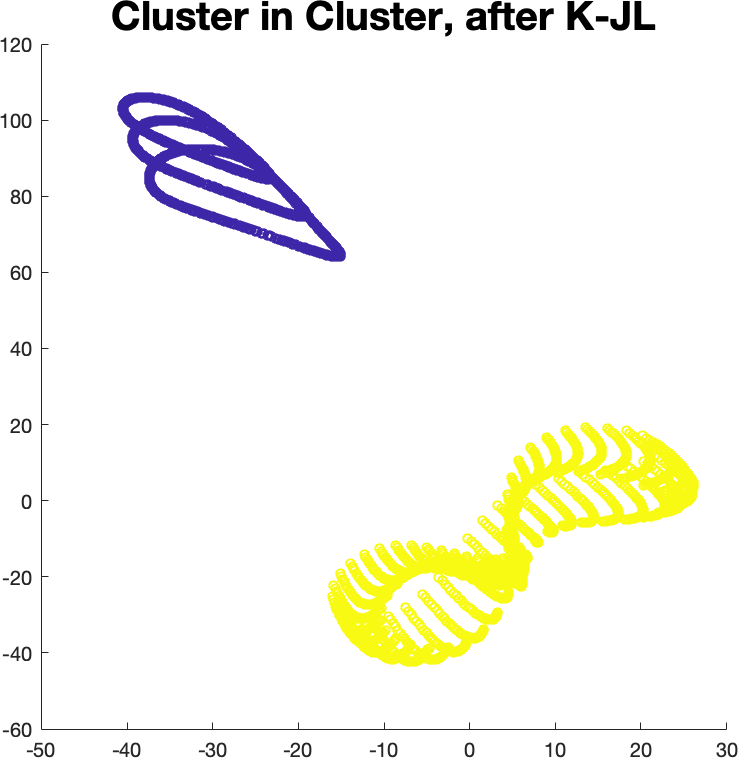}
\end{minipage}
\hspace{10mm}
\begin{minipage}[b]{0.21\textwidth}
\includegraphics[width=\textwidth]{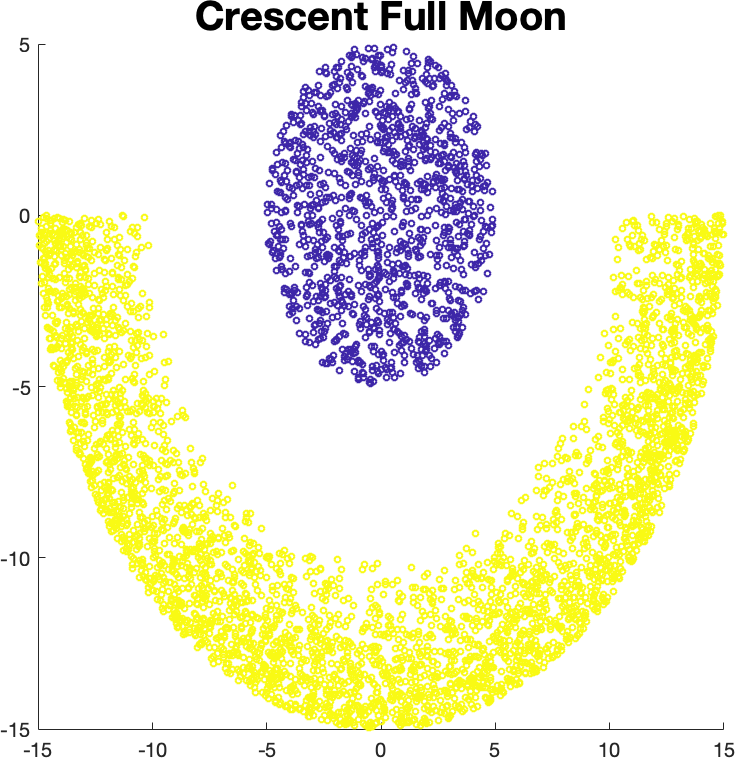}
\end{minipage}
\hspace{3mm}
\begin{minipage}[b]{0.21\textwidth}
\includegraphics[width=\textwidth]{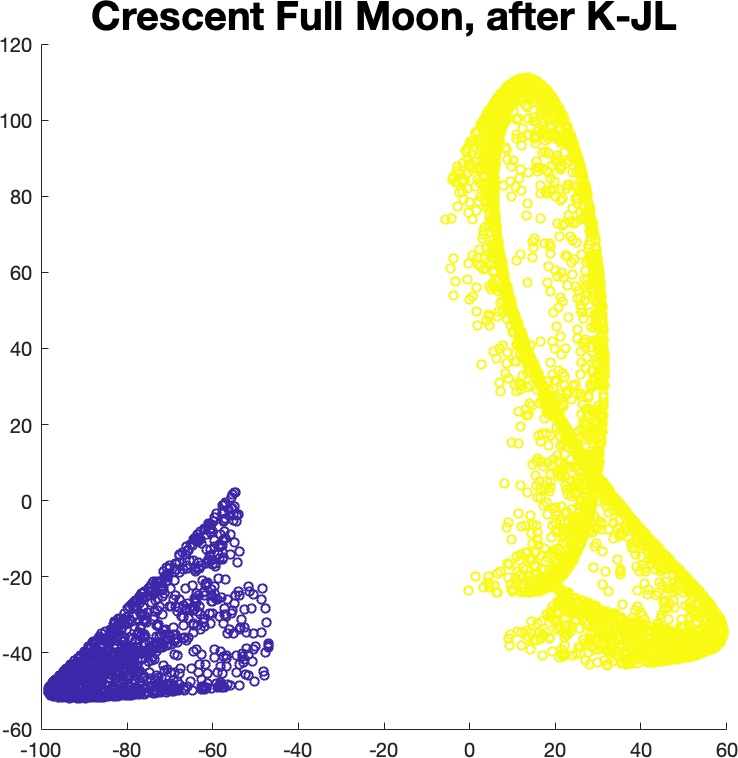}
\end{minipage}\vspace{7mm}
\caption{\small The data \emph{Cluster in Cluster} and \emph{Crescent Full Moon} each have 5000 points, and are shown before and after K-JL projection. K-JL behaves as a random version of KPCA in how it separates clusters. }
\label{fig:pca}
\vspace{-3mm}
\end{figure}

\paragraph{Simulations.} Next, we verify the above insights empirically. In particular, an empirical fact about KPCA, justifying its popularity, is that it can reveal separable subsets $\cal F, \cal G$ (in feature space) of data 
${\bm X}_N$ that were not separable in original space $\cal X$. Per the above insights, this should also be the case with K-JL. 
In Figure \ref{fig:pca} we show projection results, where, given $N = 5000$ points, we use a subsampling size $n = 100$ and projection dimension $d = 2$ to verify the intuition that K-JL (centralized) is able to separate subsets of data on typical examples (e.g., cluster in cluster) where KPCA is known to work well \citep{mika1999kernel,Scholkopf-98}.

\subsection{Preserving Clustering Properties}
In this section we argue that if the data is clusterable in feature space---an assumption underlying K-$k$-means, and 
uses of KPCA in clustering---then it remains clusterable after K-JL. 

To develop intuition, we formalize \emph{clusterability} in terms of the distribution $\rho_X$ being given as a mixture of distributions with sufficiently separated means. We adapt traditional arguments given in the work on clustering mixtures of Gaussians {\citep{dasgupta1999learning, kannan2005spectral, sanjeev2001learning}} to the square norm $\iprod{f, \Sigma^3f}_{\cal H}$. In particular, these works develop the intuition that if the $k$ cluster means are sufficiently separated, they then lie close to a $k$-dimensional subspace close to the top $k$-eigenspace of the data covariance. Such intuition holds in general Hilbert space, 
and in the sequel we illustrate this in the case of 2 clusters, while similar arguments extend to multiple clusters.

\begin{example}[Clusterability of $\rho_X$]
The following holds with probability at least $1- e^{-\tau}$, $\tau >1$. 

Let $\rho_X = \pi_1 \rho_{X, 1} + \pi_2 \rho_{X, 2}, \, 0< \pi_1, \pi_2< 1, \, \pi_1 + \pi_2 = 1$;  let $\mu_1, \mu_2$, $\Sigma_1, \Sigma_2$ are respectively the means and covariance operators of $\rho_{X, 1}, \rho_{X, 2}$, i.e., for $i =1, 2$, 
$$\mu_i = {\bb E}_{\rho_{X, i}} K(\cdot, X) \text{ and } 
\Sigma_i = {\bb E}_{\rho_{X, i}} K(\cdot, X)\otimes_{\cal H} K(\cdot, X) - \mu_i \otimes_{\cal H}\mu_i.$$

Suppose the maximum eigenvalues of $\Sigma_1, \Sigma_2$ are upper-bounded by $\sigma$. We have for some $C_1, C_2$,  both functions of $(K, \rho_X)$, that for $n\land d > C_1$: 
\begin{align}
\left \| {\hat V}^\top\sx \left(\mu_1- \mu_2\right) \right\|^2 \geq 
\lambda_1^3 \left( \|\mu_1 - \mu_2\|^2_{\cal H} - \frac{1}{\pi_1 \pi_2}\sigma \right) -  C_2\sqrt{\frac{\tau}{n\land d}}. 
\label{eq:clustering}
\end{align}
In other words, separation between cluster means are maintained. On the other hand, as a consequence of \eqref{eq:pca-ubound}, 
inter-cluster distances are maintained (at the same scale $\lambda_1^3$). 
\end{example}

The above is a consequence of the following decomposition. Let $\gamma = \pi_1/\pi_2$, so that 
$\mu_2 = \gamma \mu_1$:
\begin{align*}
\Sigma &\doteq {\bb E}_{\rho_{X}} K(\cdot, X)\otimes_{\cal H} K(\cdot, X) 
= \pi_1\left(  \Sigma_1 + \mu_1 \otimes_{\cal H}\mu_1 \right) + \pi_2\left(  \Sigma_2 + \mu_2 \otimes_{\cal H}\mu_2 \right) \nonumber \\
& = \pi_1\Sigma_1 + \pi_2 \Sigma_2 + \left(\pi_1 + \pi_2\gamma^2 \right) \mu_1 \otimes_{\cal H}\mu_1 
= \pi_1\Sigma_1 + \pi_2 \Sigma_2 + \gamma \mu_1 \otimes_{\cal H}\mu_1.
\end{align*}
It follows from the above that 
\begin{align}
\lambda_1 &\doteq \iprod{f_1, \Sigma, f_1}_{\cal H} \leq \sigma + \gamma\iprod{f_1, (\mu_1 \otimes_{\cal H}\mu_1) f_1}_{\cal H}
= \sigma + \gamma\iprod{f_1, \mu_1}_{\cal H}^2, \text{ and} \label{eq:decompsigma1}\\
\lambda_1 &\geq \frac{1}{\|\mu_1\|^2_{\cal H}}\iprod{\mu_1, \Sigma \mu_1}_{\cal H} 
\geq \gamma\frac{1}{\|\mu_1\|^2_{\cal H}}\iprod{\mu_1, \mu_1}_{\cal H}^2
= \gamma\|\mu_1\|^2_{\cal H}. \label{eq:decompsigma2}
\end{align}
Combining \eqref{eq:decompsigma1} and \eqref{eq:decompsigma2}, it follows that 
$\iprod{f_1, \mu_1}_{\cal H}^2 \geq \|\mu_1\|^2_{\cal H} -\sigma/\gamma$. Noticing that 
$\mu_1 - \mu_2 = (1+\gamma)\mu_1$, we therefore obtain 
\begin{align*}
\iprod{(\mu_1 - \mu_2), \Sigma^3(\mu_1 - \mu_2)}_{\cal H} 
&= (1+\gamma)^2 \iprod{\mu_1, \Sigma^3, \mu_1}_{\cal H} \geq (1+\gamma)^2\lambda_1^3 \iprod{f_1, \mu_1}_{\cal H}^2 \\
&\geq \lambda_1^3\left(\|\mu_1 - \mu_2\|^2_{\cal H} - \frac{(1+\gamma)^2}{\gamma}\cdot \sigma \right) 
= \lambda_1^3 \left( \|\mu_1 - \mu_2\|^2_{\cal H} - \frac{1}{\pi_1 \pi_2}\sigma \right).
\end{align*}
\pushQED{\qed}     
Equation \eqref{eq:clustering} is then obtained by combining this last inequality with Theorem \ref{Thm:main}. \qedhere
\popQED

\paragraph{Experiments.}
We run clustering experiments on UCI datasets, ${\bm X}_N$ of sizes $N$, described in Table \ref{tab:datasets}. We compare K-JL (i.e. $k$-means after centralized K-JL) against $k$-means clustering after KPCA (centralized), and K-$k$-means. For KPCA we use a fast implementation where eigen-decomposition is done on the centralized gram matrix $\bm \bar K$ of a subsample of size $n$ to approximate the top $d$ eigenfunctions of the centralized gram-matrix ${\bm \bar K}_N$ on $N$ samples; that is, if $\bm{\alpha} \in {\bb R}^n$ is an eigenvector of $\bm \bar K$, then $x\in {\bm X}_N$ is mapped to $\sum_{i \in [n]}\alpha_i K(x, x_i)$, $x_i \in \bm X$. 

{\it Nystr\"om embedding.} For K-$k$-means
we use a fast Nystr\"om embedding $\bf \tilde K_N^{1/2}$, where $\bf \tilde K_N$ approximates the gram matrix  ${\bm K}_N$ on $N$ samples, using a rank $d$ pseudo-inverse ${\bm K}^{\dagger}_{(d)}$ of the gram-matrix $\bm K$ on $n$ subsamples \citep{kmeans-nyst-18,k-means_nys,Williams-01}. That is, we use 
$\tilde {\bm K}_N = {\bm K}_{(N, n)}{\bm K}^{\dagger}_{(d)}{\bm K}_{(N, n)}^\top$, where ${\bm K}_{(N, n)}$ denotes the gram-matrix between ${\bm X}_N$ and $\bm X$.  

\begin{table}[t]
	\centering
	\caption{\small Data description}
    \label{tab:datasets}
    \begin{tabular}{|c |c |c |c|}
    \hline 
    \it UCI Datasets & \it Size $N$ & \it Dimension & \it Num. of clusters \\
    \hline \hline 
    \textsc{Avila Bible} & 20867 (bible pages) & 10  & 12 (scribes)\\
    \hline 
    \textsc{IoT} & 40000 (traffic traces) & 115& 5 (devices)\\
    \hline 
    \textsc{Bank Notes} & 1372 (images) & 4 & 2 (forged or not)\\
	\hline
    \end{tabular}
\vspace{-3mm}
\end{table}

In all our implementations, $\bm X$ are $n$ random subsamples of ${\bm X}_N$. We use a Gaussian kernel 
$K(x, x') \doteq \exp\{-\|x - x'\|^2/\sigma^2\}$, where $\sigma$ is chosen as the 25th percentile of interpoint distances.\vspace{1mm}

\emph{Relative performance}. The results of Table \ref{tab:performance} validate our intuition that K-JL achieves similar clustering as K-$k$-means and KPCA, in faster preprocessing time (for the mapping of ${\bm X}_N$, as implemented in Matlab without further optimization of matrix multiplications). For all methods, we set $d = 10k$, where $k$ is the number of clusters, and $n = \max \{200, N/100\}$. All experiments are repeated 30 times, and mean and std of Rand Index (RI) are reported. Interestingly, K-JL also appears most stable in terms of RI: the higher instability of the other two methods is likely due to the fast-eigensolvers used in Matlab. 

\begin{table}[h!]
\vspace{-3mm}
	\centering
	\caption{\small Clustering Results: Preprocessing time / Rand Index {\bf (best 2 in bold)}.}
    \label{tab:performance}
    \begin{tabular}{|c |c| c |c |c|}
    \hline 
    \it UCI Datasets &  $k$-means &  K-$k$-means & KPCA & K-JL \\
    \hline \hline 
    \textsc{Avila Bible} &
    NA / .683 $\pm$.026   & .112s / {\bf  .728} $\pm$.003  & .103s / {\bf .725} $\pm$.002 & .086s / .718 $\pm$.002\\
    \hline 
    \textsc{IoT} & 
    NA / .548 $\pm$.086 &  .537s /  .745 $\pm$.039 & .500s / {\bf .759} $\pm$.024 &  .447s / {\bf .749} $\pm$.006 \\
    \hline 
    \textsc{Bank Notes} &
    NA / .507 $\pm$.0001  &.020s / {\bf  .529} $\pm$.067 &  .014s / .526 $\pm$.041 & .007s / {\bf  .527} $\pm$.031\\
	\hline
    \end{tabular}
\end{table}

\emph{Effects of $d$ and $n$.} In Figure \ref{fig:varyingND}, we vary $n$, $d$ on the IOT dataset, where to reduce running time we now set $N = 10000$ and use 10 repetitions (rather than 30 as above) per values of $d$ and $n$. Average RI are reported. The main take-home is that the methods are most sensitive to the choice of $n$.  We again observe that Kernel-JL appears overall most stable. 

\begin{figure}[ht]
\centering
\includegraphics[height = 4.5cm]{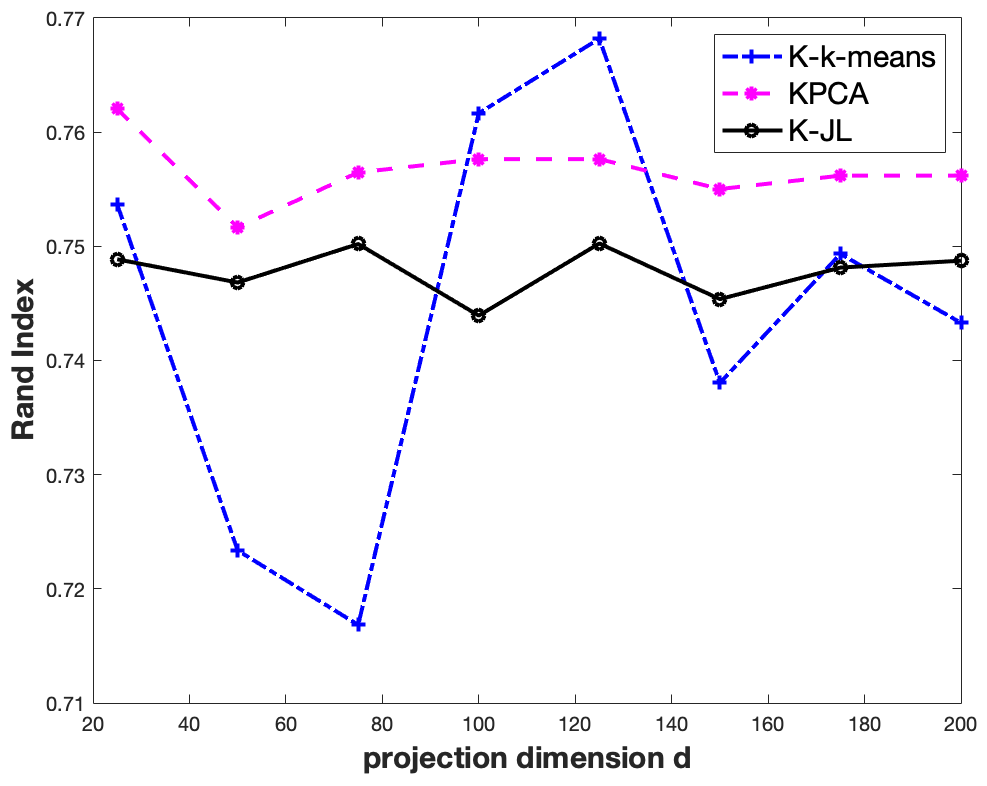}
\includegraphics[height = 4.5cm]{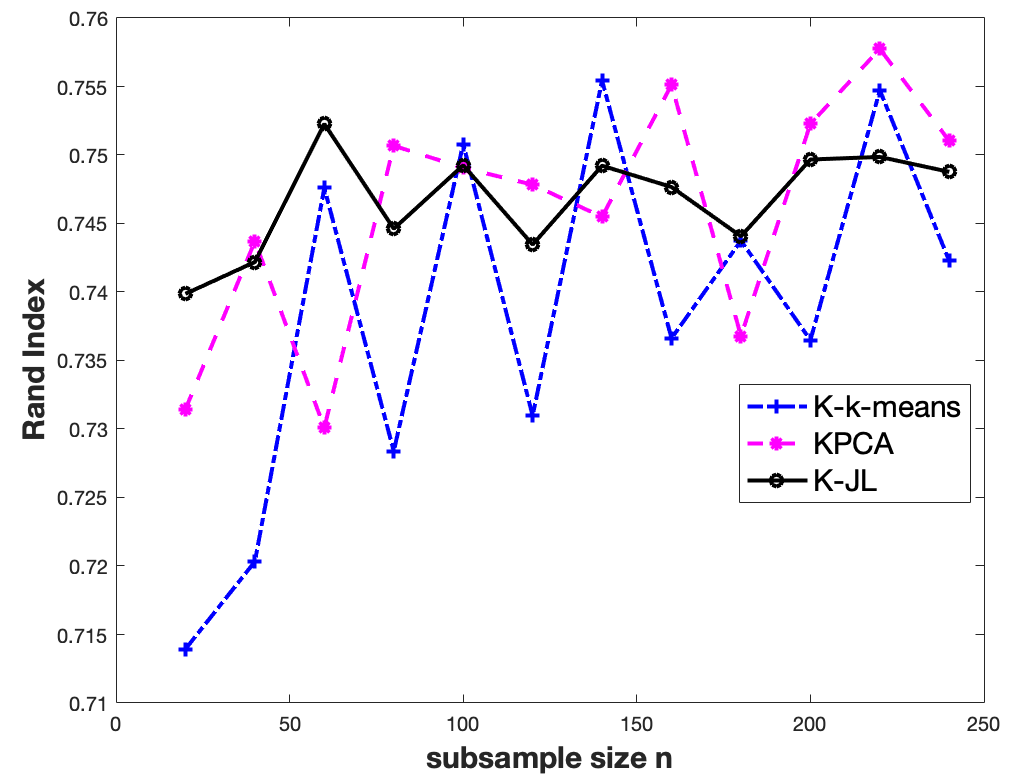}
\caption{\small Effects of $n$ and $d$, using $N =10000$ samples from the IoT dataset. Left, we fix $n = 200$ and vary $d$. Right, we fix $d = 10$ and vary $n$. The choice of subsample size $n$ seems most crucial.}
\label{fig:varyingND}
\end{figure}

\section*{Acknowledgments}
Samory Kpotufe and B. K. Sriperumbudur thank the National Science Foundation for support under, respectively,  grant id NSF-CPS-1739809, and NSF-DMS-1713011.

\bibliography{Reference}
\bibliographystyle{abbrvnat}

\appendix
\section{Proofs}
In this section, we present the missing proofs.
\subsection{Proof of Proposition~\ref{prop:Thetaconvergence}}\label{pro:sigma3}
Note that $A(f,g)\doteq\langle \Theta g,\Theta f\rangle_2-\langle g,\Sigma^3 f\rangle_\Cal{H}=\langle g,(\Theta^*\Theta -\Sigma^3)f\rangle_\Cal{H}$, where the adjoint of $\Theta$ is given by $\Theta^*:\bb{R}^d\rightarrow\Cal{H}$, $\bm{\alpha}\mapsto \frac{1}{\sqrt{d}}\sum^d_{i=1}\alpha_iv_i$. Therefore, $$\sup_{f,g\in\Cal{H}}\frac{|A(f,g)|}{\Vert f\Vert_\Cal{H}\Vert g\Vert_\Cal{H}}=\left\Vert \Theta^*\Theta-\Sigma^3\right\Vert_{\text{op}}.$$ Since $\Theta^*\Theta=\frac{1}{d}\sum^d_{i=1}v_i\otimes_{\Cal{H}}v_i$, the result is a direct application of Theorem~\ref{Thm:kL}.

\subsection{Proof of Theorem~\ref{prop:convergenceinD}}
\label{Sec:dist}
We will show that the characteristic function $\Phi_n(t)$ of $\hat V^\top S_{\bm X} f$ convergences pointwise to the characteristic 
function $\Phi(t)$ of $\Theta f$. Note that $\Phi_n(t) = \mathbb{E}_{\bm X} \phi_n(t)$, where $\phi_n(t)$ is the characteristic function of $\hat V^\top S_{\bm X} f$ conditioned on $\bm X$. 

To start, write $\Theta f=\frac{1}{\sqrt{d}}\left(\langle v_1,f\rangle_\Cal{H},\ldots,\langle v_d,f\rangle_\Cal{H}\right)^\top\sim \Cal{N}\left(0,\frac{1}{d}\langle f,\Sigma^3 f\rangle_\Cal{H}\bm{I}_d\right)$ where $\bm{I}_d$ is the $d\times d$ identity matrix. This follows by noting that $\langle v_i,f\rangle_\Cal{H}\sim\Cal{N}(0,\langle f,\Sigma^3 f\rangle_\Cal{H})$ for all $i\in[d]$ and $v_i$'s are mutually independent. Thus $\Phi(t) \doteq \exp(-\langle f,\Sigma^3 f\rangle_\Cal{H}t^2/2d)$, the characteristic function of $\Cal{N}(0,\frac{1}{d}\langle f,\Sigma^3  f\rangle_\Cal{H})$. 

Similarly, conditioned on $\bm{X}$, 
$$\hat{V}^\top\sx f=\frac{1}{n\sqrt{d}}\left(\langle \hat{v}_1,\sx f\rangle_2,\ldots,\langle \hat{v}_d,\sx f\rangle_2\right)^\top\sim\Cal{N}\left(0,\frac{1}{d}\langle f,\Sigma^3_n f\rangle_\Cal{H}\bm{I}_d\right),$$ which follows by noting that for any $i\in[d]$, 
$$\langle \hat{v}_i,\sx f\rangle_2\sim \Cal{N}\left(0,\langle \sx f,\frac{1}{n}\bm{K}^2 \sx f\rangle_2\right)=\Cal{N}\left(0,n^2\langle f,\Sigma^3_n f\rangle_\Cal{H}\right),$$
where the last equality uses the fact that $\bm{K}=\sx\sxs$, $n\Sigma_n=\sxs\sx$, and that $\hat{v}_i$'s are mutually independent. Thus, $\phi_n(t):=\exp(-\langle f,\Sigma^3_n f\rangle_{\Cal{H}}t^2/2d)$. Clearly $\langle f,\Sigma^3_nf \rangle_\Cal{H}\stackrel{a.s.}{\rightarrow}\langle f,\Sigma^3 f\rangle_\Cal{H}$ since $|\langle f,\Sigma^3_nf \rangle_\Cal{H}-\langle f,\Sigma^3 f\rangle_\Cal{H}|\le \Vert f\Vert^2_\Cal{H}\Vert \Sigma^3_n-\Sigma^3\Vert_\text{op}\stackrel{a.s.}{\rightarrow} 0$ as $n\rightarrow\infty$, which is a consequence of exponential concentration (see~\eqref{Eq:sigma3})---convergence in probability---followed by an application of Borel-Cantelli lemma. Since $y\mapsto \exp(-yt^2/2d)$ is a continuous function, it follows from Corollary 6.3.1\emph{(i)} of \cite{Resnick-14} that
$\phi_n(t) \xrightarrow{\text{a.s.}} \Phi(t)$. Since $\phi_n(t)$ is bounded, Corollary 6.3.2 of \cite{Resnick-14} implies that 
${\bb E}_{\bm X} \phi_n (t) \to \bb{E}\Phi(t)$, i.e., $\Phi_n(t)\to \Phi(t)$ for all $t\in \bb{R}$ as $n\rightarrow\infty$.
%
%

\section{Supplementary Results}
\begin{appxpro}\label{pro:gauss}
Let $(\Cal{H},k)$ be an RKHS and $\mathcal{N}_{\mathcal{H}}(0,\Sigma)$ be a Gaussian measure on $\Cal{H}$ with covariance operator $\Sigma$. Define $\tilde{w} \doteq (w(x_1),\ldots,w(x_n))^\top$ where $w\sim \mathcal{N}_{\mathcal{H}}(0,\Sigma)$. Then $$\tilde{w}\sim \mathcal{N}(0,M)$$ where $M_{jl}=\langle K(\cdot,x_j),\Sigma K(\cdot,x_l)\rangle_\Cal{H}$.
\end{appxpro}
\begin{proof}
Since $w\sim \mathcal{N}_{\mathcal{H}}(0,\Sigma)$, we have for any $g\in\Cal{H}$, $\langle g, w\rangle_\Cal{H}\sim \mathcal{N}(0,\langle g,\Sigma g\rangle_\Cal{H})$. Choosing $g=K(\cdot,x_i)$, we obtain $$w(x_i)\sim \mathcal{N}(0,\langle K(\cdot,x_i),\Sigma K(\cdot,x_i)\rangle_\Cal{H})\,\,\text{for}\,\,i\in\{1,\ldots,n\}.$$ Similarly 
\begin{eqnarray}
\text{Cov}(w(x_i),w(x_j))&{}={}&\bb{E}_w\left[w(x_i)w(x_j)\right]=\bb{E}_w\left[\langle w,K(\cdot,x_i)\rangle_\Cal{H}\langle w,K(\cdot,x_j)\rangle_\Cal{H}\right]\nonumber\\
&{}={}&\bb{E}_w\langle K(\cdot,x_i),(w\otimes_\Cal{H} w)K(\cdot,x_j)\rangle_\Cal{H}=\langle K(\cdot,x_i),\Sigma K(\cdot,x_j)\rangle_\Cal{H}\nonumber
\end{eqnarray}
and the result follows.
\end{proof}
%

\begin{appxlem}\label{lem:bern}
Suppose $\sup_{x\in\Cal{X}}K(x,x)\le\kappa$. 
Then 
$\Vert K(\cdot,x)\otimes_{\Cal{H}} K(\cdot,x)\Vert_{\emph{op}}=K(x,x)\le \kappa$ for all $x\in\Cal{X}$.
In addition, for any $\tau>0$ and $n\ge \frac{8\kappa\tau}{9\emph{tr}(\Sigma)}$, 
\begin{equation}\rho_X^n\left\{(X_1,\ldots,X_n):\Vert \Sigma_n-\Sigma\Vert_{\Cal{L}_2(\Cal{H})}\ge 4\sqrt{\frac{2\kappa\emph{tr}(\Sigma)\tau}{n}}\right\}\le 2 e^{-\tau}.\nonumber
\end{equation}
\end{appxlem}
\begin{proof}
Throughout the proof, we let $\otimes \doteq \otimes_\Cal{H}$ for ease of notation. It easily follows that
\begin{align*}
\Vert K(\cdot,x)\otimes K(\cdot,x)\Vert_{\text{op}}=&\sup_{f\in\Cal{H}}\frac{\Vert K(\cdot,x)\otimes K(\cdot,x)f\Vert_\Cal{H}}{\Vert f\Vert_\Cal{H}}\\
=& \Vert K(\cdot,x)\Vert_\Cal{H}\sup_{f\in\Cal{H}}\frac{|f(x)|}{\Vert f\Vert_\Cal{H}}=\Vert K(\cdot,x)\Vert^2_\Cal{H}=K(x,x)\le\kappa.
\end{align*}
Define $\xi_i \doteq K(\cdot,X_i)\otimes K(\cdot,X_i)-\Sigma$. It is clear that $\xi_i$'s are i.i.d.~$\Cal{L}_2(\Cal{H})$-valued random variables and $\Cal{L}_2(\Cal{H})$ is a separable Hilbert space. The result follows from a straight forward application of Bernstein's inequality in Theorem~\ref{thm:bernstein-hs} by noting that
\begin{eqnarray}
\bb{E}\Vert\xi_i\Vert^m_{\Cal{L}_2(\Cal{H})}\le\sup_{x\in\Cal{X}}\Vert K(\cdot,x)\otimes K(\cdot,x)-\Sigma\Vert^{m-2}_{\Cal{L}_2(\Cal{H})}\bb{E}\Vert\xi_i\Vert^2_{\Cal{L}_2(\Cal{H})},\nonumber
\end{eqnarray}
where
\begin{eqnarray}
\sup_{x\in\Cal{X}}\Vert  K(\cdot,x)\otimes K(\cdot,x)-\Sigma\Vert^{m-2}_{\Cal{L}_2(\Cal{H})}&{}\le{}& \sup_{x\in\Cal{X}}\left(\Vert K(\cdot,x)\otimes K(\cdot,x) \Vert_{\Cal{L}_2(\Cal{H})}+\Vert\Sigma\Vert_{\Cal{L}_2(\Cal{H})}\right)^{m-2}
\le(2\kappa)^{m-2},\nonumber
\end{eqnarray}
and \begin{equation}
\bb{E}\Vert\xi_i\Vert^2_{\Cal{L}_2(\Cal{H})}=\bb{E}\Vert K(\cdot,X_i)\otimes K(\cdot,X_i)-\Sigma\Vert^{2}_{\Cal{L}_2(\Cal{H})}
\le \bb{E}\Vert  K(\cdot,X_i)\otimes K(\cdot,X_i)\Vert^2_{\Cal{L}_2(\Cal{H})}\le \kappa\text{tr}(\Sigma).\nonumber
\end{equation}
The result follows by using $B=2\kappa$ and $\theta^2 \doteq \kappa\text{tr}(\Sigma)$ in Theorem~\ref{thm:bernstein-hs}.
\end{proof}
\begin{appxlem}\label{lem:hatU}
For any $\tau\ge 1$ and $d\ge \left(\frac{\emph{tr}(\Sigma^3_n)}{\Vert \Sigma^3_n\Vert_{\emph{op}}}\vee \tau\right)$, with probability at least $1-e^{-\tau}$ over the choice of $(\hat{u}_i)^d_{i=1}$ conditioned on $\bm{X}$, we have
\begin{equation}
\left\Vert \sxs \hat{V}\hat{V}^\top\sx-\frac{1}{n^3}\sxs\bm{K}^2\sx\right\Vert_{\emph{op}}\le\frak{C} \left(\sqrt{\frac{\emph{tr}(\Sigma^3_n)\Vert \Sigma^3_n\Vert_{\emph{op}}}{d}}+ \Vert \Sigma^3_n\Vert_{\emph{op}}\sqrt{\frac{\tau}{d}}\right),\nonumber
\end{equation}
where $\frak{C}$ is a universal constant independent of $\bm{K}$ and $d$.
\end{appxlem}
\begin{proof}
Note that $\bb{E}\left(\sxs\hat{V}\hat{V}^\top \sx\right)=\sxs\bb{E}(\hat{V}\hat{V}^\top)\sx=\frac{1}{n^2d}\sxs\sum^d_{i=1}\bb{E}(\hat{v}_i\hat{v}^\top_i)\sx=\frac{1}{n^2}\sxs\bb{E}(\hat{v}_1\hat{v}^\top_1)\sx=\frac{1}{n^3}\sxs\bm{K}^2\sx=\frac{1}{n^3}\sxs\sx\sxs\sx\sxs\sx=\Sigma^3_n$ where we used the facts that $\bm{K}=\sx\sxs$ and $\Sigma_n=\frac{1}{n}\sxs\sx$. Therefore conditioning on $\bm{X}$, it follows from Theorem~\ref{Thm:kL} that for all $\tau\ge 1$ and $d\ge (r(\Sigma^3_n)\vee \tau)$, with probability $1-e^{-\tau}$ over the choice of $(\hat{v}_i)^d_{i=1}$, we obtain
\begin{eqnarray}
\left\Vert \sxs \hat{V}\hat{V}^\top\sx -\frac{1}{n^3}\sxs\bm{K}^2 \sx\right\Vert_{\text{op}}&{}\le{}&\mathfrak{C}
 \frac{1}{n^2}\Vert \Sigma^3_n\Vert_{\text{op}}\left(\sqrt{\frac{r(\Sigma^3_n)}{d}}+ \sqrt{\frac{\tau}{d}}\right)\nonumber,
\label{Eq:hat-term}
\end{eqnarray}
where
$r(\Sigma^3_n)\le\frac{\text{tr}(\Sigma^3_n)}{\Vert \Sigma^3_n\Vert_{\text{op}}}.$ 
\end{proof}
\section{Known Technical Tools}\label{sec:additional}
\begin{appxthm}[\citealp{Koltchinskii-17}]\label{Thm:kL}
Let $X_1,\ldots,X_n$ be i.i.d.~centered Gaussian random variables in a separable Hilbert space $H$ with covariance operator $\Sigma=\bb{E}[X\otimes_H X]$. Let $\hat{\Sigma}=\frac{1}{n}\sum_{i=1}X_i\otimes_H X_i$ be the empirical covariance operator. Define $$r(\Sigma) \doteq \frac{\left(\bb{E}\Vert X\Vert_{H}\right)^2}{\Vert \Sigma\Vert_{\emph{op}}}.$$ Then for all $\tau\ge 1$ and $n\ge (r(\Sigma)\vee \tau)$, with probability at least $1-e^{-\tau}$,
$$\Vert \hat{\Sigma}-\Sigma\Vert_\emph{op}\le \frak{C}\Vert \Sigma\Vert_{\emph{op}}\frac{\sqrt{r(\Sigma)}+\sqrt{\tau}}{\sqrt{n}},$$
where $\frak{C}$ is a universal constant independent of $\Sigma$, $\tau$ and $n$.
\end{appxthm}

%
    
    
    
    
    
The following is Bernstein's inequality in separable Hilbert spaces which is quoted from \cite[Theorem 3.3.4]{yurinsky95sums}.

\begin{appxthm}[Bernstein's inequality]\label{thm:bernstein-hs}
Let $(\Omega,\mathcal{A},P)$ be a probability space, $H$ be a separable Hilbert space, $B>0$ and $\theta>0$. Furthermore, let $\xi_1,\ldots,\xi_n:\Omega\rightarrow H$ be zero mean i.i.d.~random variables satisfying 
\begin{equation}\bb{E}\Vert \xi_1\Vert^m_{H}\le \frac{m!}{2}\theta^2B^{m-2},\,\,\forall\,\,m>2.\nonumber
\end{equation}
Then for any $\tau>0$,
$$P^n\left\{(\xi_1,\ldots,\xi_n):\left\Vert\frac{1}{n} \sum^n_{i=1}\xi_i\right\Vert_H\ge \frac{2B\tau}{n}+\sqrt{\frac{2\theta^2\tau}{n}}\right\}\le 2 e^{-\tau}.$$
\end{appxthm}
    
\end{document}